\documentclass[runningheads]{llncs}
\usepackage[latin9]{inputenc}
\usepackage{amsmath}
\usepackage{amssymb}
\usepackage{graphicx}

\makeatletter

\providecommand{\tabularnewline}{\\}

\usepackage{times}
\usepackage{latexsym}

\usepackage{color}
\usepackage{units}
\usepackage{comment}

\makeatother

\begin{document}
\title{Tractable Fragments of Temporal Sequences of Topological Information}
\author{Quentin Cohen-Solal}
\institute{LAMSADE, Universit\'e Paris-Dauphine, PSL, CNRS, France\\
\texttt{\textbf{quentin.cohen-solal@dauphine.psl.eu}}}

\maketitle
\global\long\def\R{\mathfrak{R}}%
\global\long\def\I{\mathtt{I}}%
\global\long\def\NP{\mathrm{NP}}%
\global\long\def\hier{\mathfrak{h}}%
\global\long\def\E{\mathtt{E}}%
\global\long\def\S{\mathcal{S}}%
\global\long\def\ca{\mathfrak{C}}%
\global\long\def\U{\mathcal{U}}%
\global\long\def\T{\mathcal{T}}%
\global\long\def\TS{\mathfrak{T}}%
\global\long\def\P{\mathcal{P}}%
\global\long\def\Nrcc{\mathcal{N}}%
\global\long\def\NPrcc{\mathcal{NP}_{8}}%
\global\long\def\Prcc{\mathcal{P}_{8}}%
\global\long\def\A{\mathcal{A}}%
\global\long\def\F{\mathcal{F}}%
\global\long\def\B{\mathcal{B}}%
\global\long\def\Base{\boldsymbol{\mathcal{B}}}%
\global\long\def\cbase{\widehat{\Base}}%
\global\long\def\a#1{\mathrm{a}_{#1}}%
\global\long\def\PA{\mathrm{PA}}%
\global\long\def\IA{\mathrm{IA}}%
\global\long\def\HH{\mathcal{H}_{8}}%
\global\long\def\CH{\mathcal{C}_{8}}%
\global\long\def\QH{\mathcal{Q}_{8}}%
\global\long\def\ERA{\mathrm{ERA}}%
\global\long\def\RCCH{\mathrm{RCC}_{8}}%

\global\long\def\RCCAH{\mathrm{RCA}_{8}}%
\global\long\def\TRCCHt{\mathrm{T\RCCH^{s}}}%
\global\long\def\TRCCHv{\mathrm{T\RCCH^{n}}}%
\global\long\def\TRCCHd{\mathrm{T\RCCH^{d}}}%
\global\long\def\RCCHc{\mathcal{C}_{\RCCH}}%
\global\long\def\RCCHmax{\mathcal{S}_{\RCCH}}%
\global\long\def\ERAc{\mathcal{C}_{\ERA}}%
\global\long\def\ERAmax{\mathcal{S}_{\ERA}}%
\global\long\def\PAc{\mathcal{C}_{\PA}}%
\global\long\def\PAmax{\mathcal{S}_{\PA}}%
\global\long\def\IAc{\mathcal{C}_{\IA}}%
\global\long\def\IAmax{\mathcal{S}_{\IA}}%
\global\long\def\pp{{\scriptstyle \bigtriangleup}}%
\global\long\def\ppi{{\scriptstyle \bigtriangledown}}%
\global\long\def\tp{\overset{{\scriptstyle \bigtriangleup}}{{\scriptstyle \bigtriangledown}}}%
\global\long\def\rb{b}%
\global\long\def\rpb{\cdot b}%
\global\long\def\rbp{b\cdot}%
\global\long\def\rpbp{<}%
\global\long\def\rs{s}%
\global\long\def\rps{\cdot s}%
\global\long\def\rf{f}%
\global\long\def\rpf{\cdot f}%
\global\long\def\ro{o}%
\global\long\def\rm{m}%
\global\long\def\rd{d}%
\global\long\def\rpd{\cdot d}%
\global\long\def\re{e}%
\global\long\def\rpe{=}%
\global\long\def\rdi{\bar{d}}%
\global\long\def\rpdi{\cdot\bar{d}}%
\global\long\def\roi{\bar{o}}%
\global\long\def\rmi{\bar{m}}%
\global\long\def\rsi{\bar{s}}%
\global\long\def\rpsi{\cdot\bar{s}}%
\global\long\def\rfi{\bar{f}}%
\global\long\def\rpfi{\cdot\bar{f}}%
\global\long\def\rbi{\bar{b}}%
\global\long\def\rpbi{\cdot\bar{b}}%
\global\long\def\rbip{\bar{b}\cdot}%
\global\long\def\rpbip{>}%
\global\long\def\conv{\mathop{\Rsh}}%
\global\long\def\upconv{\mathop{\uparrow}}%
\global\long\def\downconv{\mathop{\downarrow}}%
\global\long\def\convde#1{{_{#1}\Rsh}}%
\global\long\def\c{\mathrm{C}}%
\global\long\def\dr{\mathrm{DR}}%
\global\long\def\pp{\mathrm{PP}}%
\global\long\def\o{\mathrm{O}}%
\global\long\def\p{\mathrm{P}}%
\global\long\def\dc{\mathrm{DC}}%
\global\long\def\ec{\mathrm{EC}}%
\global\long\def\po{\mathrm{PO}}%
\global\long\def\tpp{\mathrm{TPP}}%
\global\long\def\ntpp{\mathrm{NTPP}}%
\global\long\def\tppi{\mathrm{\overline{TPP}}}%
\global\long\def\ppi{\mathrm{\overline{PP}}}%
\global\long\def\ntppi{\mathrm{\overline{NTPP}}}%
\global\long\def\eq{\mathrm{EQ}}%
\global\long\def\convinv{\{\Rsh^{-1}\}}%
\global\long\def\inv#1{\overline{#1}}%
\global\long\def\comp{\circ}%
\global\long\def\vide{\varnothing}%
\global\long\def\contrainte#1#2#3{#1\mathrel{#2}#3}%
\global\long\def\and{\thinspace\wedge\thinspace}%
\global\long\def\ou{\thinspace\vee\thinspace}%
\global\long\def\liste#1#2{\left\{  #1\,|\,#2\right\}  }%
\global\long\def\TRCCHvcp{\mathrm{T\RCCH^{n}}[\eq\nleftrightarrow\ntpp]}%
\global\long\def\TRCCHdcp{\mathrm{T\RCCH^{d}}[\eq\nleftarrow\ntpp]}%
\global\long\def\Sntpp{\S^{{\scriptscriptstyle \ntpp\Rightarrow\tpp}}}%
\global\long\def\Sntpppo{\S^{{\scriptscriptstyle \ntpp\Rightarrow\po}}}%
\global\long\def\Qntpp{\mathcal{Q}_{8}^{{\scriptscriptstyle \ntpp\Rightarrow\tpp}}}%
\global\long\def\Qntpppo{\mathcal{Q}_{8}^{{\scriptscriptstyle \ntpp\Rightarrow\po}}}%
\global\long\def\Hntpp{\HH^{{\scriptscriptstyle \ntpp\Rightarrow\tpp}}}%
\global\long\def\Hntpppo{\HH^{{\scriptscriptstyle \ntpp\Rightarrow\po}}}%
\global\long\def\vois{\mathop{\updownarrow}}%
\global\long\def\implique{\,\implies\,}%
\global\long\def\cquatre#1#2#3#4{(#1,#2,#3,#4)}%
\global\long\def\cdeux#1#2{\left(#1,#2,\right)}%
\global\long\def\ctrois#1#2#3{\left(#1,#2,#3\right)}%
\global\long\def\liste#1#2{\left\{  #1\,|\,#2\right\}  }%

\begin{abstract}
In this paper, we focus on qualitative temporal sequences of topological
information. We firstly consider the context of topological temporal
sequences of length greater than $3$ describing the evolution of
regions at consecutive time points. We show that there is no Cartesian
subclass containing all the basic relations and the universal relation
for which the algebraic closure decides satisfiability. However, we
identify some tractable subclasses, by giving up the relations containing
the non-tangential proper part relation and not containing the tangential
proper part relation.

We then formalize an alternative semantics for temporal sequences.
We place ourselves in the context of the topological temporal sequences
describing the evolution of regions on a partition of time (i.e. an
alternation of instants and intervals). In this context, we identify
large tractable fragments. 

\end{abstract}

\section{Introduction}

The reasoning on temporal and spatial qualitative information is necessary
to solve many problems that are found in the context of planning,
simulation, robotics, intelligent environments and human-computer
interaction~\cite{landsiedel2017review,cohn1998exploiting,westphal2011guiding,dylla2004exploiting,mansouri2016robot,sioutis2017towards}.
For this reason, many \emph{spatio-temporal formalisms} have been
proposed~\cite{gerevini2002qualitative,ragni2006temporalizing,westphal2013transition,sioutis2014qualitative,sioutis2015ordering,sioutis2015generalized,cohen2017temporal,ligozat2013qualitative,chen2015survey,dylla2017survey}.
Spatio-temporal formalisms generally decompose into a spatial formalism
and a temporal formalism. The point algebra is a formalism describing
the relative positions of points on a line (the timeline or a line
of space). $\RCCH$ is another formalism, more expressive than the
point algebra, expressing the \emph{topological relations} between
regions. It expresses the notions of contact and inclusion.

The \emph{qualitative temporal sequences} \cite{westphal2013transition,cohen2017temporal}
are the simplest spatio-temporal descriptions, in the sense that there
is no uncertainty about temporal information. However, strong negative
results have been identified for one of the simplest spatial formalisms:
deciding the satisfiability of a temporal sequence over the point
algebra is $\NP$-complete (even while restricting the language to
\emph{basic relations} and \emph{the universal relation})~\cite{westphal2013transition}.
One can then wonder whether deciding satisfiability is necessarily
$\NP$-hard within the framework of spatio-temporal formalisms. However,
the complexity of fragments of $\RCCH$ has not been studied within
the context of temporal sequences. There could be fragments, not expressing
the point algebra, which are tractable.

We therefore study in this paper the complexity of deciding the satisfiability
of the topological temporal sequences. We identify in particular a
negative result: the classical procedure to decide the satisfiability
of polynomial fragments, the \emph{algebraic closure}, does not decide
the satisfiability in this context (even while being limited to the
basic relations and to the universal relation, if the length of the
sequence is greater than $3$). We also identify a positive result,
by considering semantics different from the classical semantics of
temporal sequences. More precisely, we no longer consider that temporal
sequences describe the evolutions of entities at neighboring instants.
We consider instead that they describe the evolutions of entities
on a partition of time (i.e. on an alternation of instants and intervals).
In the context of this semantics, we identify large tractable fragments.

In the next section, we present related work, the $\RCCH$ formalism
and temporal sequences over $\RCCH$. In Section~\ref{sec:-temporaliser-sans-relations},
we introduce our negative result and we identify tractable fragments
that do not contain all the basic relations. Finally, in Section~\ref{sec:-temporaliser-sur-partition},
we formalize the alternative semantics of temporal sequences, present
the new reasoning operators, and then identify the large tractable
fragments.

Note, this paper is an extended version of \cite{cohen2020tractable}.

\section{Background}

\subsection{Related Work}

Many works deal with spatio-temporal reasoning and its complexity.
Tractable fragments have been identified in the context of topological
temporal sequences describing the evolution of \emph{constant-size}
regions at non-\emph{neighboring} instants (regions can satisfy any
relations between the instants)~\cite{cohen2017temporal}. The temporal
sequences that we consider, like those of the $\NP$-completeness
result of the point algebra, describe regions at neighboring instants.
\emph{Temporal sequence ordering} (at neighboring instants) is an
$\NP$-complete problem for several fundamental formalisms, such as
$\RCCH$~\cite{sioutis2015ordering}. Formalisms with a higher temporal
expressivity have also been proposed. For example, $\RCCH$ has been
combined with \emph{Allen's interval algebra}~\cite{gerevini2002qualitative}.
The cardinal direction calculus has also been combined with the Allen's
interval algebra~\cite{ragni2006temporalizing}.

In general, a qualitative spatio-temporal formalism is based on a
transition graph, i.e. a graph representing the possible evolutions
of basic relations. It can be a \emph{neighbourhood graph}~\cite{freksa1991conceptual}
or a \emph{dominance graph}~\cite{galton2001dominance,galton2000qualitative}.
In a neighbourhood graph, two relations $b$, $b'$ are \emph{neighbour}
(i.e. adjacent in the graph), if there exists a pair of evolving entities
$(e,e')$ satisfying $b$ at an instant $t$ and $b'$ at an instant
$t'$, and satisfying $b$ or $b'$ between $t$ and $t'$. In a dominance
graph, a relation $b$ dominates another relation $b'$ (i.e. there
is an arc from $b'$ to $b$ in the graph), denoted by $b\vdash b'$,
if there exists $t,t'\in\mathbb{R}$ and a pair of evolving entities
$(e,e')$ satisfying $b$ at $t$ and satisfying $b'$ at each instant
of $]t,t']$. Many transition graphs have been determined~\cite{santos2009conceptual,wu2014towards,van2005conceptual,ligozat1994towards,freksa1991conceptual,egenhofer2010family,egenhofer2015qualitative,dylla2007qualitative,zimmermann1993enhancing,freksa1992temporal,cohn2001qualitative,reis2008conceptual,ragni2008reasoning,kurata20079+,reis2008conceptual,ragni2006temporalizing}.

Spatio-temporal qualitative reasoning is also studied in the context
of logics (see \cite{hazarika2001qualitative,hazarika2005qualitative,bennett2002multi,muller2002topological,sioutis2015generalized,galton1993towards,bennett2000describing,wolter2000spatio,wolter2002qualitative,gabelaia2003computational,burrieza2005multimodal,burrieza2011pdl}).
Deciding the satisfiability of these logics is generally PSPACE-hard.
Ontologies of time based on points or/and intervals have been studied~\cite{van2013logic}.
There is, in particular, the Event Calculus, a logic of action and
change, which can express properties at instant and interval \cite{kowalski1989logic}.

\subsection{Region Connection Calculus $\protect\RCCH$}

$\RCCH$ \cite{randell1992spatial,li2003region,ligozat2013qualitative}
is a classical qualitative formalism~\cite{ligozat2013qualitative,chen2015survey,dylla2017survey}.
Thus, it is a triplet $\left(\A,\U,\varphi\right)$ where $\A$ is
a set of relations forming a finite \emph{non-associative binary relation
algebra}, $\U$ is the \emph{universe}, i.e. the set of considered
entities, and $\varphi$ is a particular interpretation function associating
with each relation of $\A$ a relation over $\U$. We denote by $\RCCAH$
the algebra of $\RCCH$. The universe $\U$ of $\RCCH$ is the set
of regions of a certain topological space $\TS$ (i.e. the non-empty,
\emph{closed} and \emph{regular} subsets of $\TS$). $\TS$ is generally
$\mathbb{R}^{n}$. Any algebra $\A$ has special relations, called
\emph{basic relations}. Every relation of $\A$ is a union of basic
relations. The $8$ basic relations of $\RCCH$: $\dc$ (disconnected),
$\ec$ (externally connected), $\po$ (partially overlapping), $\eq$
(equal), $\tpp$ (tangential proper part), $\ntpp$ (non-tangential
proper part), and the converse of the two previous relations are described
in Figure~\ref{fig:RCC8_relations_horizontal} and defined in Table~\ref{tab:def-relations-rcc8}.
We denote by $\Base_{\RCCH}$ the set of the basic relations of $\RCCH$.
We denote by $\B_{\RCCH}$ the \emph{universal relation} (i.e. the
union of all relations) and by $\vide$ the empty relation. Any algebra
$\A$ has several operators: the \emph{union} $\cup$, the \emph{intersection}
$\cap$, the \emph{converse} $\inv{\cdot}$, and the \emph{(abstract)
composition} $\comp$. These operators are used to infer new relations:
$\contrainte xry\implies\contrainte y{\inv r}x$, $\contrainte xry\and\contrainte x{r'}y\implies\contrainte x{\left(r\cap r'\right)}y$,
and $\contrainte xry\and\contrainte y{r'}z\implies\contrainte x{\left(r\comp r'\right)}z$
(with $r,r'\in\A$ and $x,y,z$ being entity variables). The abstract
composition $\comp$ of $\RCCH$ is the \emph{weak composition}: $r\comp r'=\bigcup\liste{b\in\Base_{\RCCH}}{\varphi\left(b\right)\cap\left(\varphi\left(r\right)\comp\varphi\left(r'\right)\right)\neq\vide}$
with $r,r'\in\A$. For example, the composition of relations $\tpp\cup\eq$
and $\tpp$ is the relation $\tpp$. The composition of basic relations
is described in a so-called \emph{composition table}~\cite{li2003region}.

\begin{table}
\begin{centering}
\begin{tabular}{|c|c|}
\hline 
Relation & Definition\tabularnewline
\hline 
\hline 
$\contrainte x{\dc}y$ & $\lnot\left(\contrainte x{\c}y\right)$\tabularnewline
\hline 
$\contrainte x{\p}y$ & $\forall z\quad\contrainte z{\c}x\implies\contrainte z{\c}y$\tabularnewline
\hline 
$\contrainte x{\pp}y$ & $\contrainte x{\p}y\and\lnot\left(\contrainte y{\p}x\right)$\tabularnewline
\hline 
$\contrainte x{\eq}y$ & $\contrainte x{\p}y\and\contrainte y{\p}x$\tabularnewline
\hline 
$\contrainte x{\o}y$ & $\exists z\quad\contrainte z{\p}x\and\contrainte z{\p}y$\tabularnewline
\hline 
$\contrainte x{\po}y$ & $\contrainte x{\o}y\and\lnot\left(\contrainte x{\p}y\right)\and\lnot\left(\contrainte y{\p}x\right)$\tabularnewline
\hline 
$\contrainte x{\ec}y$ & $\contrainte x{\c}{y\and\lnot\left(\contrainte x{\o}y\right)}$\tabularnewline
\hline 
$\contrainte x{\tpp}y$ & $\contrainte x{\pp}{y\and\left(\exists z\ \contrainte z{\ec}x\and\contrainte z{\ec}y\right)}$\tabularnewline
\hline 
$\contrainte x{\ntpp}y$ & $\contrainte x{\pp}{y\and\lnot\left(\exists z\ \contrainte z{\ec}x\and\contrainte z{\ec}y\right)}$\tabularnewline
\hline 
$\contrainte x{\tppi}y$ & $\contrainte y{\tpp}x$\tabularnewline
\hline 
$\contrainte x{\ntppi}y$ & $\contrainte y{\ntpp}x$\tabularnewline
\hline 
\end{tabular}
\par\end{centering}
\caption{Definitions of RCC relations ($C(x,y)$ is the \emph{contact relation},
it means that the closedregions $x$ and $y$ intersect ; $P$ is
the \emph{part} \emph{relation} ; $\protect\o$ the \emph{overlap
relation} ; variables $x,y,z$ are closed regions).\label{tab:def-relations-rcc8}}
\end{table}

Generally, a description based on $\RCCH$ is a qualitative constraint
network, i.e. a conjunction of relations between different entities.
For instance, $\contrainte x{\dc\cup\ec}y\and\contrainte z{\tpp\cup\ntpp\cup\eq}y$
is such a description, which means that the interiors of regions $x$
and $y$ are disjoint and that the region $z$ is included in $y$.
Deciding the satisfiability of qualitative constraint networks whose
relations belong to $\RCCAH$ is an $\NP$-complete problem~\cite{renz1999maximal}.
Tractable fragments have been identified. They consist in restricting
the relations of constraint networks to a particular subset of $\RCCAH$.
Three large tractable subsets containing all the basic relations and
the universal relation have been identified: $\HH$, $\QH$, and $\CH$~\cite{renz1999maximal}.
They are moreover maximal for tractability. They are defined in Table~\ref{tab:Definitions-de-Q8,H8,C8}.
On these subclasses, applying the \emph{algebraic closure}, which
is a reasoning operator on networks using the algebra operators, decides
satisfiability.

\begin{table*}
\begin{centering}
\begin{tabular}{|c|c|}
\hline 
 & Definition\tabularnewline
\hline 
\hline 
$\Nrcc$ & $\liste{r\in\RCCAH}{\po\nsubseteq r\and r\cap\left(\tpp\cup\ntpp\right)\neq\vide\and r\cap\left(\tppi\cup\ntppi\right)\neq\vide}$\tabularnewline
\hline 
$\NPrcc$ & $\Nrcc\cup\liste{r_{1}\cup\ec\cup r_{2}\cup\eq}{r_{1}\in\left\{ \vide,\dc\right\} \and r_{2}\in\left\{ \ntpp,\ntppi\right\} }$\tabularnewline
\hline 
$\Prcc$ & $\RCCAH\backslash\NPrcc$\tabularnewline
\hline 
$\HH$ & $\Prcc\cap\liste{r\in\RCCAH}{\ntpp\cup\eq\subseteq r\implies\tpp\subseteq r\and\ntppi\cup\eq\subseteq r\implies\tppi\subseteq r}$\tabularnewline
\hline 
$\QH$ & $\Prcc\cap\liste{r\in\RCCAH}{\left(\eq\subseteq r\and r\cap\left(\tpp\cup\ntpp\cup\tppi\cup\ntppi\right)\neq\vide\right)\implies\po\subseteq r}$\tabularnewline
\hline 
$\CH$ & $\Prcc\cap\liste{r\in\RCCAH}{\left(\ec\subseteq r\and r\cap\left(\tpp\cup\ntpp\cup\tppi\cup\ntppi\cup\eq\right)\neq\vide\right)\implies\po\subseteq r}$\tabularnewline
\hline 
\end{tabular}
\par\end{centering}
\caption{\label{tab:Definitions-de-Q8,H8,C8}Definitions of the relations sets
$\protect\HH$, $\protect\QH$, and $\protect\CH$.}
\end{table*}

\begin{figure}[t]
\begin{centering}
\includegraphics[scale=0.22]{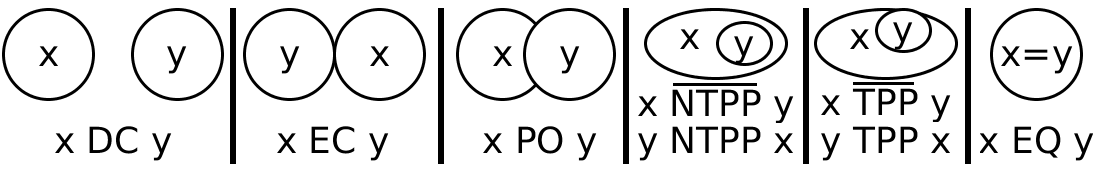}
\par\end{centering}
\centering{}\caption{The $8$ basic relations of $\protect\RCCH$ in the plane.}
\label{fig:RCC8_relations_horizontal}
\end{figure}

\subsection{Link with finite CSP}

We briefly discuss the links between qualitative formalisms and finite
CSP. On the one hand, for some qualitative formalisms, the algebraic
closure enforces path-consistency \cite{renz2005weak}. On the other
hand, a qualitative constraints network can be translated into a network
of finite quantitative constraints \cite{westphal2009qualitative}.
The CSP variables are the relations between the qualitative variables.
More precisely, there is a CSP variable $v_{xy}$ for each pair of
qualitative variables $(x,y)$. The set of possible values for the
CSP variable $v_{xy}$ is the set of basic relations contained in
the relation between $x$ and $y$. The CSP constraints between the
CSP variables encode the composition operator. These constraints are
ternary and of the form $\liste{(b'',b,b')\in\Base^{3}}{b''\subseteq b\comp b'}$.

\subsection{Semantics of continuously evolving regions\label{subsec:Semantics-of-continuously}}

Before presenting temporal sequences over $\RCCH$, we must formally
define what we call a region evolving continuously over time. A region
evolving continuously during a time interval $I$ (i.e. a real closed
interval) is naturally defined as a continuous function $f$ from
$I$ to the set of considered regions $\R$ of a topological space
(for instance, $\R$ can be the regions of $\mathbb{R}^{n}$ with
$n\geq1$ and can possibly be restricted to convex or connected regions).
However, this standard mathematical definition requires that $\R$
be associated with a topology. Thus, we require the following concept:
\begin{definition}
A\emph{ topological region space} $(\R,T)$ is a set of regions $\R$
of a topological space associated with a topology $T$ (i.e. $\left(\R,T\right)$
is also a topological space).
\end{definition}

There are several possible topologies for the regions of $\mathbb{R}^{n}$~\cite{galton2000qualitative,davis2001continuous}.
In particular, choosing a metric between regions amounts to choosing
a topology. Depending on the choice of the topological region space,
the evolution of regions satisfies or violates certain properties,
such as continuity of particular functions (area, distance, union,
projection, convex hull, ...)~\cite{davis2001continuous}. In fact,
solids, gases, shadows, ... do not evolve continuously in the same
way~\cite{galton2000qualitative}. The usual metric of the regions
of $\mathbb{R}^{n}$ is the \emph{Hausdorff distance}. Unfortunately,
the corresponding evolution of the relations of regions is not \emph{compatible}
with the classical neighbourhood graph of $\RCCH$ (Figure~\ref{fig:Graphe-de-voisinage}.a)~\cite{davis2001continuous}.
The \emph{dual-Hausdorff distance}~\cite{davis2001continuous} corrects
this problem: the evolution of regions according to this metric is
compatible with the classical neighbourhood graph of $\RCCH$. Note
that other metrics also correct it.

\begin{figure}
\begin{centering}
\includegraphics[width=0.8\columnwidth]{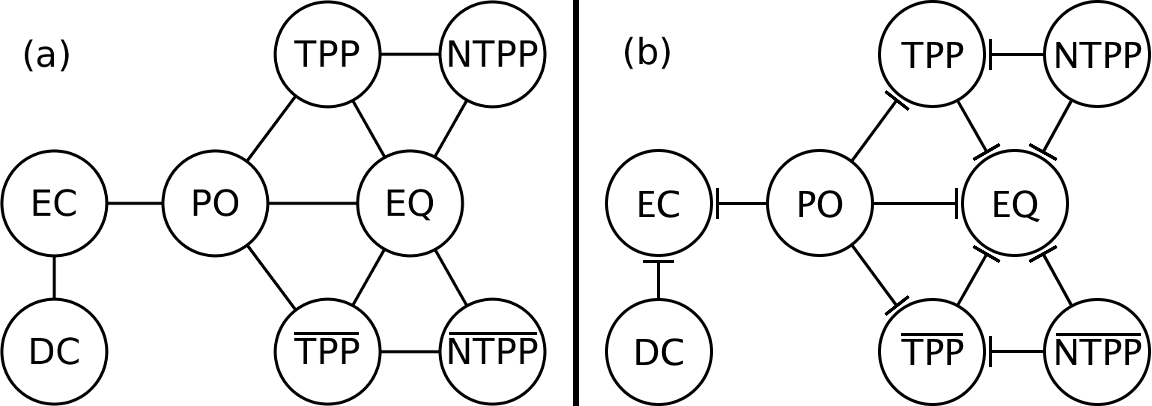}
\par\end{centering}
\caption{\label{fig:Graphe-de-voisinage}Neighbourhood graph of $\protect\RCCH$
(a) and dominance graph of $\protect\RCCH$ (b).}
\end{figure}

\subsection{\label{subsec:Temporalisations-de-RCC8}Topological Sequences at
Neighboring Instants}

We present in this section the topological temporal sequences describing
the continuous evolution of regions at neighboring instants~\cite{cohen2017temporal},
that we denote $\TRCCHv$. For this, we recall the basics of the framework
of \emph{multi-algebras}~\cite{cohen2017checking,cohen2017temporal}
from which it is defined. It is an abstract framework that includes
several extensions of classical qualitative formalisms, such as temporal
sequences.

\subsubsection{Projections and Relations}

Multi-algebras generalize non-associative binary relation algebras.
\emph{A multi-algebra} is a Cartesian product $\A=\A_{1}\times\cdots\times\A_{m}$
of relation algebras satisfying certain properties. We denote by $\I$
the index set of the multi-algebra, i.e. $\left\{ 1,\ldots,m\right\} $.
In the context of temporal sequences, each algebra $\A_{i}$ corresponds
to the same relation algebra but to a different time period. The \emph{set
of basic relations} of $\A$, denoted $\Base$, is $\Base_{1}\times\cdots\times\Base_{m}$
where $\Base_{i}$ is the set of basic relations of $\A_{i}$. Multi-algebras
are equipped with a set of additional operators $\conv_{i}^{j}$ from
$\A_{i}$ to $\A_{j}$ for all distinct $i,j\in\I$, called \emph{projections}.
Any projection $\conv$ satisfies by definition $\conv\left(r\cup r'\right)=\left(\conv r\right)\cup\left(\conv r'\right)$
and $\conv\left(\inv r\right)=\inv{\conv\left(r\right)}$. In the
context of temporal sequences, projections describe the possible evolution
of relations over time. 
\begin{table}
\begin{centering}
\begin{tabular}{|c|c|}
\hline 
$b$ & $\vois b$\tabularnewline
\hline 
\hline 
$\dc$ & $\dc\cup\ec$\tabularnewline
\hline 
$\ec$ & $\dc\cup\ec\cup\po$\tabularnewline
\hline 
$\po$ & $\ec\cup\po\cup\tpp\cup\tppi\cup\eq$\tabularnewline
\hline 
$\tpp$ & $\po\cup\tpp\cup\ntpp\cup\eq$\tabularnewline
\hline 
$\ntpp$ & $\tpp\cup\ntpp\cup\eq$\tabularnewline
\hline 
$\eq$ & $\po\cup\tpp\cup\ntpp\cup\tppi\cup\ntppi\cup\eq$\tabularnewline
\hline 
\end{tabular}
\par\end{centering}
\caption{Neighboring relations of the basic relations of $\protect\RCCH$.\label{tab:Correspondances-des-relations}}
\end{table}

\begin{definition}
The operator $\vois$ from $\RCCAH$ to $\RCCAH$ is the projection
satisfying the Table~\ref{tab:Correspondances-des-relations}.
\end{definition}

The projection $\vois$ encodes the neighbourhood graph of $\RCCH$
described in Figure~\ref{fig:Graphe-de-voisinage}.a (i.e. $b\subseteq\vois b'$
if and only if $b$ and $b'$ are neighbours). For instance, $\po\nsubseteq\vois\dc$
because it is not possible to have a continuous transition from $\dc$
to $\po$. A \emph{relation} of a multi-algebra is an $m$-tuplet
of classical relations. By adding semantics, that is to say a universe
$\U$ and a specific interpretation function $\varphi$, we get a
qualitative formalism said \emph{loosely combined}, also called \emph{sequential
formalism}. 
\begin{example}
To illustrate the preceding concepts and to give intuition concerning
$\TRCCHv$, we give some examples ($\TRCCHv$ will be formalized in
the next subsection). The Cartesian product of the multi-algebra $\A$
of $\TRCCHv$ is $\RCCAH^{m}$ where $m$ is the length of considered
temporal sequences (the sequences describe regions at instants $t_{1},\ldots,t_{m}$).
The component $i\in\I$ of a relation $R$ of $\TRCCHv$, $R_{i}$,
is the relation of $\RCCH$ which must be satisfied at the instant
$t_{i}$. An example of relations of $\TRCCHv$, with $m=3$, is $\left(\tpp\cup\ntpp\cup\tppi\cup\ntppi,\po\cup\eq,\ec\cup\dc\right)$.
This relation means on the one hand that one of the two regions is
first included in the other ($R_{1}$ is satisfied at $t_{1}$), then
they overlap or are equal ($R_{2}$ is satisfied at $t_{2}$) and
finally they are disjoint ($R_{3}$ is satisfied at $t_{3}$). Since
the instants of the sequence are \emph{neighbors}, this relation means,
on the other hand, that between $t_{1}$ and $t_{2}$ the regions
satisfy either the basic relation being satisfied at $t_{1}$ (which
is $\tpp$ or $\ntpp$ or $\tppi$ or $\ntppi$) or the one being
satisfied at $t_{2}$ (which is $\po$ or $\eq$) and between $t_{2}$
and $t_{3}$ they satisfy either the basic relation being satisfied
at $t_{2}$ (which is $\po$ or $\eq$) or the one being satisfied
at $t_{3}$ (which is $\ec$ or $\dc$). This additional constraint
is called \emph{continuity without intermediary relation} and also
\emph{continuous qualitative change~}\cite{cohen2017temporal,westphal2013transition}.
This constraint enforces that each sequence of relations $R$ describes
all changes of relations between regions. In other words, between
two instants $t_{i}$ and $t_{i+1}$, there must be no change of relation,
other than the transition from the basic relation satisfied at $t_{i}$
towards the basic relation satisfied at $t_{i+1}$. 
\end{example}

\subsubsection{Multi-algebra and Relation Operators}

We recall the definition of $\TRCCHv$ (originally denoted $\mathrm{TT_{wir}}$~\cite{cohen2017temporal}). 

\begin{definition}
\label{exa:La-topologie-temporaliser} Let $t_{1},\ldots,t_{m}\in\mathbb{R}$
be consecutive instants and $(\R,T)$ be a topological region space.
$\TRCCHv$ is the triplet $\left(\A,\U,\varphi\right)$ where:
\begin{itemize}
\item $\A$ is $\RCCAH^{m}$ equipped with the projections $\conv_{i}^{j}$
fully defined by $\conv_{i}^{j}b=\vois b$ if $\left|j-i\right|=1$
and $\conv_{i}^{j}b=\B_{\RCCH}$ otherwise, for all $b\in\Base_{\RCCH}$
and $i,j\in\I$, 
\item $\U$ is the set of continuous functions from $[t_{1},t_{m}]$ to
$\R$, and
\item $\varphi$ is the function from $\A$ to $2^{\U\times\U}$ such that
for all $R\in\A$, $\varphi\left(R\right)$ is the set of pairs of
functions $\left(f,f'\right)\in\U\times\U$ satisfying at each instant
$t_{i}$ the relation $R_{i}$ (i.e. $\forall i\in\I\ \left(f\left(t_{i}\right),f'\left(t_{i}\right)\right)\in\varphi_{\RCCH}\left(R_{i}\right)$)
and satisfying no intermediary relations between each instants $t_{i}$
and $t_{i+1}$ (i.e. during each $[t_{i},t_{i+1}]$ a basic relation
is satisfied, then another, formally: for all $i\in\I\backslash\left\{ m\right\} $
there exist $\tau\in[t_{i},t_{i+1}]$ and $b,b'\in\Base_{\RCCH}$
such that at each instant $t\in[t_{i},\tau[$, $\left(f\left(t\right),f'\left(t\right)\right)\in\varphi_{\RCCH}\left(b\right)$,
at each instant $t\in]\tau,t_{i+1}]$, $\left(f\left(t\right),f'\left(t\right)\right)\in\varphi_{\RCCH}\left(b'\right)$,
and that $\left(f\left(\tau\right),f'\left(\tau\right)\right)\in\varphi_{\RCCH}\left(b\cup b'\right)$),
with $\varphi_{\RCCH}$ the interpretation function of $\RCCH$.
\end{itemize}
\end{definition}

\begin{remark}
\label{sec:Discussion-sur-la} $\TRCCHv$ depends on a set of regions
but also on a topology for the regions (i.e. a notion of continuity).
Note that $\TRCCHv$ is not necessarily a sequential formalism (i.e.
its reasoning operators are not necessarily correct: reasoning operators
could remove some solutions). To be a sequential formalism, the evolution
of regions corresponding to the chosen topological region space must
be compatible with the classical neighbourhood graph of $\RCCH$ (see
Section~\ref{subsec:Semantics-of-continuously}). Thus, if $\R$
is $\mathbb{R}^{n}$ equipped with the dual-Hausdorff distance, $\TRCCHv$
is a sequential formalism.
\end{remark}

Every multi-algebra has operators on its relations, namely \emph{composition},
\emph{union}, \emph{intersection}, and \emph{converse}. They are defined
componentwise. For example, the composition of $R$ and $R'$, $R\comp R'$,
is defined by $\left(R\comp R'\right)_{i}=R_{i}\comp R_{i}'$ for
all $i\in\I$.

There is another operator on relations: the \emph{projection closure}
of a relation $R$, denoted $\conv\left(R\right)$. It consists in
sequentially applying the following operation until reaching a fixed
point: for all $j\in\I$, $R_{j}\leftarrow R_{j}\cap\bigcap_{i\neq j}\conv_{i}^{j}R_{i}$.
Projection closure refines relations by removing classical basic relations
that are impossible to satisfy. In the context of $\TRCCHv$, projection
closure enforces continuity without intermediary relation. For example,
the projection closure of the following relation, with $m=3$, $\left(\tpp\cup\ntpp\cup\tppi\cup\ntppi,\po\cup\eq,\ec\cup\dc\right)$
is $\left(\tpp\cup\tppi,\po,\ec\right)$. Indeed, in particular, there
is no transition from the relation $\po$ or from the relation $\eq$
to the relation $\dc$ without intermediary relation. In addition,
there is no transition from $\ntpp$ to $\dc$ or $\ec$ in just two
qualitative changes. Projection closure removes such impossibilities.
Relations closed under projection, i.e. satisfying $\conv\left(R\right)=R$,
are said \emph{$\conv$-closed}. Note that projection closure can
be seen as a kind of arc-consistency.

\subsubsection{Constraint Networks and Algebraic Closure}

A description in the context of multi-algebras is a \emph{(qualitative
constraint) network}. A network over a multi-algebra $\A$ is a set
of variables $\E$ and a function $N$ associating with each pair
of variables $(x,y)\in\E^{2}$ such that $x\neq y$ a relation of
$\A$ and satisfying $N(x,y)=\inv{N(y,x)}$ for all distinct $x,y\in\E$.
A sequence of classical constraint networks is thus represented by
a single constraint network whose relations are sequences of relations,
i.e. relations of a multi-algebra. We denote $N(x,y)$ more succinctly
by $N^{xy}$. It is sometimes useful to refer to the ``subnetwork''
corresponding to the index $i\in\I$ of a network $N$, denoted $N_{i}$,
called \emph{slice}. $N_{i}$ is defined by $\left(N_{i}\right)^{xy}=\left(N^{xy}\right)_{i}$
for all distinct $x,y\in\E$. In the context of temporal sequences,
the slice $i$ of a network $N$, $N_{i}$, describes the relations
of the sequence at the instant $t_{i}$. Similarly, the \emph{slice}
$i\in\I$ of a subset $\S\subseteq\A$, denoted $\S_{i}$, is $\liste{R_{i}}{R\in\S}$.
A network is said to be \emph{satisfiable} (or \emph{consistent})
if there is a \emph{solution} to this network, that is, an assignment
for the variables $\left\{ u_{x}\right\} _{x\in\E}\subseteq\U$ satisfying
the relations of the constraint network, i.e. $\left(u_{x},u_{y}\right)\in\varphi\left(N^{xy}\right)$.
A network $N$ is said \emph{over} a subset of relations $\S\subseteq\A$
if for all distinct $x,y\in\E$, $N^{xy}\in\S$. A \emph{scenario}
is a network over $\Base$.

The reasoning operator on networks is the \emph{algebraic closure},
which applies the operators of the multi-algebra. It propagates information
within the network, makes inferences, by refining relations. In the
context of topological temporal sequences, the algebraic closure propagates
information over regions at each instant and between the different
instants. A relation $R$ \emph{refines} a relation $R'$ if $R_{i}\subseteq R'_{i}$
for all $i\in\I$. More generally, $N$ \emph{refines} $N'$, denoted
$N\subseteq N'$, if for all distinct $x,y\in\E$, $N^{xy}\subseteq\left(N'\right)^{xy}$.
Algebraic closure closes networks under composition and under projection.
Algebraic closure thus applies the two following operations until
reaching a fixed point: $N^{xz}\leftarrow N^{xz}\cap\left(N^{xy}\comp N^{yz}\right)$
and $N^{xz}\leftarrow\conv\left(N^{xz}\right)$ for all distinct $x,y,z\in\E$.
We denotes by $\ca\left(N\right)$ the algebraic closure of $N$.
In the context of topological temporal sequences, the composition
operator makes inferences at every instant and the projection operator
makes inferences between the instants. A network $N$ is called \emph{algebraically
closed} if it is closed under composition, i.e. for all distinct $x,y,z\in\E$,
$N^{xz}\subseteq N^{xy}\comp N^{yz}$, and if each of its relations
$N^{xy}$ is closed under projection, i.e. for all distinct $i,j\in\I$,
$N_{j}^{xy}\subseteq\conv_{i}^{j}N_{i}^{xy}$. 

\subsubsection{\emph{Consistency and Satisfiability}}

A sequential formalism is said \emph{complete} if all its algebraically
closed scenarios are satisfiable. It is a fondamental property for
deciding satisfiability in an algebraic way. 
\begin{remark}
To know if $\TRCCHv$ is complete for the regions of $\mathbb{R}^{n}$
equipped with the dual-Hausdorff distance or for another topological
region space is a complex problem. In fact, perhaps there is no topological
region space such that $\TRCCHv$ is a complete sequential formalism.
For this reason, several studies have dealt with a ``weak satisfiability'',
i.e. satisfiability with a weaker notion of continuity~\cite{gerevini2002qualitative,westphal2013transition,sioutis2015ordering,bennett2002multi}.
Formally, a network $N$ over $\TRCCHv$ is \emph{weakly satisfiable}
if it contains an algebraically closed scenario (i.e. if there exists
a sequence of satisfiable classical scenarios satisfying the constraints
of the networks $N_{i}$ \emph{and} the neighbourhood graph). The
tractable subclasses that we identify in this article are tractable
for $\TRCCHv$ associated with a topological region space such that
$\TRCCHv$ is a complete sequential formalism. They are also tractable
for this notion of weak satisfiability.
\end{remark}

A relation $R$ is said \emph{trivially unsatisfiable} if there exists
$i\in\I$ such that $R_{i}=\vide$. Note that a relation which is
not trivially unsatisfiable can be \emph{unsatisfiable}, i.e. $\varphi\left(R\right)=\vide$.
This is the case of $\left(\po,\po,\dc\right)$. A relation is said
\emph{$\conv$-consistent} if it is $\conv$-closed and it is not
trivially unsatisfiable. A network is said \emph{trivially unsatisfiable
}if there exists distinct $x,y\in\E$ such that $N^{xy}$ is a trivially
unsatisfiable relation. An algebraically closed network that is not
trivially unsatisfiable is said to be \emph{algebraically consistent}. 

\subsubsection{Tractable Subclasses}

By restricting networks to certain subsets of relations $\S$, we
get the following property: if the algebraic closure of a network
over $\S$ is algebraically consistent, then this network is satisfiable.
Such subsets are said to be \emph{algebraically tractable}. In other
words, with an algebraically tractable subset $S$, to decide the
satisfiability of a network over $S$, it suffices to verify that
its algebraic closure is not trivially inconsistent. The search for
algebraically tractable subsets has focused on particular subsets~\cite{ligozat2013qualitative}.
A subset $\S\subseteq\A$ is said a \emph{subclass} if it is closed
under intersection, composition, and converse (i.e. for all $R,R'\in\S$,
we have $R\cap R'\in\S$, $R\comp R'\in\S$, and $\inv R\in\S$).\emph{
}Subclasses containing all basic relations (i.e. $\Base\subseteq\S$)
are called \emph{subalgebras}. A subset $\S\subseteq\A$ is said \emph{$\conv$-closed}
if for all $R\in\S$, $\conv(R)\in\S$. Finally, we say that a subset
$\S$ is \emph{Cartesian} if $\S=\S_{1}\times\cdots\times\S_{m}$.

Note that a list of conditions guaranteeing algebraic tractability
has been identified~\cite{cohen2017checking} (see the slicing and
refinement theorems). One of these conditions is algebraic stability
by a refinement $H$. A \emph{refinement} $H$ is a function from
$\A$ to $\A$ satisfying $H\left(R\right)\subseteq R$. A subset
$\S\subseteq\A$ is \emph{algebraically stable by $H$} if for any
algebraically consistent network $N$ over $\S$, the network $H\left(N\right)$
is algebraically consistent, where $H\left(N\right)$ is the network
obtained from $N$ by substituting each relation $N^{xy}$ by $H\left(N^{xy}\right)$.
$\RCCH$ has two fundamental refinements: $h_{\HH}\left(r\right)=\a{\tppi}\left(\a{\tpp}\left(\a{\po}\left(\a{\ec}\left(\a{\dc}\left(r\right)\right)\right)\right)\right)$
and $h_{\CH}\left(r\right)=\a{\tppi}\left(\a{\tpp}\left(\a{\ntppi}\left(\a{\ntpp}\left(\a{\po}\left(\a{\dc}\left(r\right)\right)\right)\right)\right)\right)$
with $\a b$, $b\in\Base$, the function from $\RCCAH$ to $\RCCAH$
defined by $\a b\left(r\right)=b$ if $b\subseteq r$ and $\a b\left(r\right)=r$
otherwise. $\HH$ and $\QH$ are algebraically stable by $h_{\HH}$
and $\CH$ is algebraically stable by $h_{\CH}$~\cite{renz1999maximal}.
Moreover, for every $r\in\HH\cup\QH$ such that $r\neq\vide$, $h_{\HH}\left(r\right)\in\Base_{\RCCH}$
and for every $r\in\CH\backslash\left\{ \vide\right\} $, $h_{\CH}\left(r\right)\in\Base_{\RCCH}$~\cite{renz1999maximal}.
In the following, we are interested in the refinement $H_{\S}$ defined
by $H_{\S}\left(R\right)_{i}=h_{\HH}\left(R_{i}\right)$ if $\S_{i}\subseteq\HH$
or $\S_{i}\subseteq\QH$ and $H_{\S}\left(R\right)_{i}=h_{\CH}\left(R_{i}\right)$
otherwise, for all $i\in\I$ and $R\in\RCCAH^{m}$, with $\S\subseteq\RCCAH^{m}$.

\section{Study of $\protect\TRCCHv$ Subclasses\label{sec:-temporaliser-sans-relations}}

In this section, we are interested in temporalized $\RCCH$ at neighboring
instants, i.e. $\TRCCHv$ (see Section~\ref{subsec:Temporalisations-de-RCC8}).
More precisely, we search for subclasses that are algebraically tractable.
Unfortunately, as the following proposition shows, there are no algebraically
tractable Cartesian subalgebras (at least for $m\geq4$). 
\begin{proposition}
Let $\R$ be a topological region space such that $\TRCCHv$ is a
sequential formalism.

No Cartesian subalgebra of $\TRCCHv$ is algebraically tractable (when
$m\geq4$).

No $\conv$-closed Cartesian subalgebra of $\TRCCHv$ is algebraically
tractable (if $m\geq2$). 
\end{proposition}

\begin{proof}
We show the case $m=4$. The idea is that the algebraic closure can
produce relations $R$ verifying $\tpp\cup\eq\subseteq R_{i}\subseteq\tpp\cup\ntpp\cup\eq$
and $\tppi\cup\eq\subseteq R_{j}\subseteq\tppi\cup\ntppi\cup\eq$
with $\left|i-j\right|=1$ which causes that some unsatisfiable networks
are algebraically consistent. Let $\S$ be any Cartesian subalgebra
(thus $\S$ contains the closure of the basic relations and the universal
relation of $\RCCH$ under intersection, composition, and converse).
We show that there exists an unsatisfiable network over $\S$ whose
algebraic closure is algebraically consistent. Consider the network
$N$ satisfying: $\E=\left\{ u,v,w,x,y,z\right\} $, $N_{1}^{xy}=\ntpp$,
$N_{4}^{xy}=\ntppi$, $N_{1}^{xz}=\ntpp$, $N_{3}^{wz}=\ntpp$, $N_{1}^{yz}=\tppi$,
$N_{2}^{yz}=\po\cup\tpp\cup\tppi\cup\eq$, $N_{1}^{wx}=\tppi$, $N_{2}^{wx}=\po\cup\tpp\cup\tppi\cup\eq$,
$N_{2}^{wy}=\po\cup\tpp$, $N_{4}^{xu}=\ntppi$, $N_{2}^{vu}=\ntppi$,
$N_{4}^{yu}=\tpp$, $N_{3}^{yu}=\po\cup\tpp\cup\tppi\cup\eq$, $N_{4}^{vx}=\tpp$,
$N_{3}^{vx}=\po\cup\tpp\cup\tppi\cup\eq$, $N_{3}^{vy}=\po\cup\tppi$,
and $N_{i}^{ab}=\B_{\RCCH}$ in the other cases. The network is over
$\S$ (indeed, we have $\dc\comp\dc=\B_{\RCCH}$, $\tppi\comp\tpp=\po\cup\tpp\cup\tppi\cup\eq$,
and $\left(\ec\comp\ec\right)\cap\left(\ec\comp\ntpp\right)=\po\cup\tpp$).
Its algebraic closure $\ca\left(N\right)$ is the algebraically consistent
network satisfying: 
\begin{itemize}
\item $\ca\left(N\right)^{xy}=\cquatre{\ntpp}{\tpp\cup\ntpp\cup\eq}{\tppi\cup\ntppi\cup\eq}{\ntppi}$,
\item $\ca\left(N\right)^{yz}=\cquatre{\tppi}{\po\cup\tppi\cup\eq}{\B\backslash\left(\dc\cup\ec\right)}{\B\backslash\dc}$,
\item $\ca\left(N\right)^{xz}=\cquatre{\ntpp}{\tpp\cup\ntpp\cup\eq}{\B\backslash\left(\dc\cup\ec\right)}{\B\backslash\left(\dc\cup\ec\right)}$,
\item $\ca\left(N\right)^{wx}=\cquatre{\tppi}{\po\cup\tppi\cup\eq}{\B\backslash\left(\dc\cup\ntppi\right)}{\B}$, 
\item $\ca\left(N\right)^{wy}=\cquatre{\po\cup\tpp\cup\ntpp}{\po\cup\tpp}{\B\backslash\left(\dc\cup\ntppi\right)}{\B}$,
\item $\ca\left(N\right)^{wz}=\cquatre{\po\cup\tpp\cup\ntpp}{\tpp\cup\ntpp\cup\eq}{\ntpp}{\tpp\cup\ntpp\cup\eq}$, 
\item $\ca\left(N\right)^{ux}=\cquatre{\B\backslash\dc}{\B\backslash\left(\dc\cup\ec\right)}{\tpp\cup\ntpp\cup\eq}{\ntpp}$, 
\item $\ca\left(N\right)^{uy}=\cquatre{\B\backslash\left(\dc\cup\ec\right)}{\B\backslash\left(\dc\cup\ec\right)}{\po\cup\tppi\cup\eq}{\tppi}$,
\item $\ca\left(N\right)^{uz}=\cquatre{\B\backslash\left(\dc\cup\ec\right)}{\B\backslash\left(\dc\cup\ec\right)}{\B\backslash\dc}{\B\backslash\dc}$, 
\item $\ca\left(N\right)^{uw}=\cquatre{\B\backslash\dc}{\B}{\B}{\B}$, $\ca\left(N\right)^{vx}=\cquatre{\B\backslash\dc}{\B\backslash\left(\dc\cup\ec\cup\ntpp\right)}{\po\cup\tpp\cup\eq}{\tpp}$, 
\item $\ca\left(N\right)^{vy}=\cquatre{\B\backslash\left(\dc\cup\ec\right)}{\B\backslash\left(\dc\cup\ec\cup\ntpp\right)}{\po\cup\tppi}{\ec\cup\po\cup\tppi\cup\ntppi}$, 
\item $\ca\left(N\right)^{vz}=\cquatre{\B\backslash\left(\dc\cup\ec\right)}{\B\backslash\left(\dc\cup\ec\right)}{\B\backslash\dc}{\B}$, 
\item $\ca\left(N\right)^{vw}=\cquatre{\B\backslash\dc}{\B}{\B}{\B}$, 
\item $\ca\left(N\right)^{vu}=\cquatre{\tppi\cup\ntppi\cup\eq}{\ntppi}{\tppi\cup\ntppi\cup\eq}{\po\cup\tppi\cup\ntppi}$.
\end{itemize}
 However, $N$ is not satisfiable since $\ca\left(N\right)$ is not
satisfiable. Indeed, to refine $\ca\left(N\right)_{2}^{xy}$ by $\eq$,
$\tpp$, or $\ntpp$ and then to apply the algebraic closure gives
a trivially unsatisfiable network. This can be seen by the fact that
the only satisfiable basic relations $B$ of $\ca\left(N\right)^{xy}$
satisfy $B_{2}=\eq$ or $B_{3}=\eq$ and that there exists neither
algebraically closed scenario $S\subseteq\ca\left(N\right)_{2}$ satisfying
$S^{xy}=\eq$ nor algebraically closed scenario $S\subseteq\ca\left(N\right)_{3}$
satisfying $S^{xy}=\eq$. For example, by setting $\ca\left(N\right)_{2}^{xy}=\eq$,
we get $\ca\left(N\right)_{2}^{xy}=\ca\left(N\right)_{2}^{yz}=\ca\left(N\right)_{2}^{xz}=\eq$
and therefore $\ca\left(N\right)_{2}^{wx}=\ca\left(N\right)_{2}^{wy}=\ca\left(N\right)_{2}^{wz}=\ca\left(N\right)_{2}^{wx}\cap\ca\left(N\right)_{2}^{wy}\cap\ca\left(N\right)_{2}^{wz}$
i.e. $\left(\po\cup\tppi\cup\eq\right)\cap\left(\po\cup\tpp\right)\cap\left(\tpp\cup\ntpp\cup\eq\right)=\vide$.

Let $\S$ be a $\conv$-closed Cartesian subalgebra of $\TRCCHv$
with $m=2$. The network $N'$ satisfying $\E=\left\{ u,v,w,x,y,z\right\} $,
$N'_{1}=\ca\left(N\right)_{2}$, and $N'_{2}=\ca\left(N\right)_{3}$
is algebraically consistent, unsatisfiable, and over $\S$. 
\end{proof}

For this reason, we are looking for Cartesian subclasses that do not
contain all basic relations. In particular, we are interested in the
following subset of $\RCCAH$ : $\Hntpp$ defined by $\liste{r\in\HH}{\ntpp\subseteq r\implies\tpp\subseteq r\and\ntppi\subseteq r\implies\tppi\subseteq r}$.
It is easy to prove that this subset is a subclass.
\begin{lemma}
\label{lem:ss-classe}The subset $\Hntpp$ is a subclass.
\end{lemma}

We show that we can obtain algebraically tractable Cartesian subclasses
of $\TRCCHv$ satisfying $\S_{i}\in\left\{ \QH,\Hntpp\right\} $.
For this, we apply the refinement theorem (by using the refinement
$H_{\S}$ ; see Section~\ref{subsec:Temporalisations-de-RCC8}).
We begin by showing the conditions of the theorem. 
\begin{lemma}
\label{lem:voisinage-clos}We have the following properties: 
\begin{itemize}
\item $\forall r\in\QH$, $\vois r\in\Hntpp$ and 
\item $\forall r\in\Hntpp$, $\vois r\in\QH$. 
\end{itemize}
\end{lemma}

\begin{proof}
From the definitions of $\HH$, $\QH$, and the projections of $\TRCCHv$,
we derive the lemma. We show on the one hand that for any $r\in\QH$,
we have $\vois r\in\Hntpp$. Let $r\in\QH$. If $\po\nsubseteq\vois r$
then $r\subseteq\dc\cup\ntpp\cup\ntppi$ (only $\dc$, $\ntpp$, and
$\ntppi$ satisfy $\po\nsubseteq\vois b$ with $b\in\B_{\RCCH}$).
Therefore, $r$ is $\dc$, $\ntpp$, $\ntppi$, $\dc\cup\ntpp$, or
$\dc\cup\ntppi$ (since $r\in\QH$). Thus, $\vois r\in\Hntpp$. Suppose
$\po\subseteq\vois r$. If $\ntpp\nsubseteq\vois r$ and $\ntppi\nsubseteq\vois r$
then $\vois r\in\Hntpp$. Otherwise, if $\ntpp\subseteq\vois r$ then
either $\tpp\subseteq r$ or $\ntpp\subseteq r$ or $\eq\subseteq r$.
In all cases, $\tpp\subseteq\vois r$. Similarly, if $\ntppi\subseteq\vois r$
then $\tppi\subseteq\vois r$. Therefore, in all cases, we have $\vois r\in\Hntpp$.

We show on the other hand that for any $r\in\Hntpp$, we have $\vois r\in\QH$.
Let $r\in\Hntpp$. If $\po\nsubseteq\vois r$ then $r\subseteq\dc\cup\ntpp\cup\ntppi$
and therefore $r=\dc$. Thus, $\vois r=\dc\cup\ec$ and therefore
$\vois r\in\QH$. If $\po\subseteq\vois r$, we have $\vois r\in\QH$
by definition.
\end{proof}

\begin{lemma}
\label{lem:voisinage-stable}Let $\S\in\left\{ \QH\times\QH,\QH\times\HH,\HH\times\QH\right\} $
be a subclass of $\TRCCHv$ $(m=2)$ and let $R\in\S$.

If $R$ is $\conv$-consistent then $H_{\S}\left(R\right)$ is $\conv$-consistent. 
\end{lemma}

\begin{proof}
From the definitions of $\HH$, $\QH$, $h_{\HH}$, and the projections
of $\TRCCHv$, we derive the lemma. Let $\S\in\left\{ \QH\times\QH,\QH\times\HH,\HH\times\QH\right\} $
and $R\in\S$ a $\conv$-consistent relation. We show that $H_{\S}\left(R\right)$
is $\conv$-consistent. 
\begin{itemize}
\item If $\dc\subseteq R_{1}$ then $R_{2}\cap\left(\dc\cup\ec\right)\neq\vide$
and therefore $H_{\S}\left(R\right)=\left(\dc,\dc\right)$ or $H_{\S}\left(R\right)=\left(\dc,\ec\right)$. 
\item Otherwise, if $\ec\subseteq R_{1}$ then $R_{2}\cap\left(\dc\cup\ec\cup\po\right)\neq\vide$
and therefore $H_{\S}\left(R\right)\in\left\{ \left(\ec,\dc\right),\left(\ec,\ec\right),\left(\ec,\po\right)\right\} $. 
\item Otherwise, if $\po\subseteq R_{1}$ then $R_{2}\cap\left(\ec\cup\po\cup\tpp\cup\eq\cup\tppi\right)\neq\vide$
and $\dc\nsubseteq R_{2}$. 
\begin{itemize}
\item If $R_{2}\cap\left(\ec\cup\po\cup\tpp\cup\tppi\right)\neq\vide$,
then $H_{\S}\left(R\right)_{1}=\po$ and $H_{\S}\left(R\right)_{2}\in\left\{ \ec,\po,\tpp,\tppi\right\} $. 
\item Otherwise, $\eq\subseteq R_{2}$ and $R_{2}\subseteq\eq\cup\ntpp\cup\ntppi$
and therefore $R_{2}=\eq$ ($R_{2}\in\HH\cup\QH$). Thus, $H_{\S}\left(R\right)=\left(\po,\eq\right)$. 
\end{itemize}
\item Otherwise, if $\tpp\subseteq R_{1}$ then $R_{2}\cap\left(\po\cup\tpp\cup\ntpp\cup\eq\right)\neq\vide$,
$R_{2}\cap\left(\dc\cup\ec\right)=\vide$ and $R_{1}\subseteq\tpp\cup\ntpp\cup\eq$
($R_{1}\in\HH\cup\QH$). 
\begin{itemize}
\item If $R_{2}\cap\left(\po\cup\tpp\right)\neq\vide$ then $H_{\S}\left(R\right)$
is either $\left(\tpp,\po\right)$ or $\left(\tpp,\tpp\right)$. 
\item Otherwise, $R_{2}\subseteq\ntpp\cup\eq\cup\tppi\cup\ntppi$. 
\begin{itemize}
\item If $\ntpp\subseteq R_{2}$ then $R_{2}=\ntpp$ ($R_{2}\in\HH\cup\QH$)
and therefore $H_{\S}\left(R\right)=\left(\tpp,\ntpp\right)$. 
\item If $\eq\subseteq R_{2}$ then $R_{2}\subseteq\tppi\cup\ntppi\cup\eq$
($R_{2}\in\HH\cup\QH$). 
\begin{itemize}
\item If $\S_{1}=\QH$ then $R_{1}\subseteq\tpp\cup\ntpp$ and therefore
$R_{2}=\eq$. Thus, $H_{\S}\left(R\right)=\left(\tpp,\eq\right)$. 
\item If $\S_{2}=\QH$ then $R_{2}=\eq$ and therefore $H_{\S}\left(R\right)=\left(\tpp,\eq\right)$. 
\end{itemize}
\end{itemize}
\end{itemize}
\item Otherwise, if $\tppi\subseteq R_{1}$ then $R_{2}\cap\left(\dc\cup\ec\right)=\vide_{2}$,
$R_{2}\cap(\po\cup\tppi\cup\ntppi\cup\eq)\neq\vide$ and $R_{1}\subseteq\tppi\cup\ntppi\cup\eq$
($R_{1}\in\HH\cup\QH$). 
\begin{itemize}
\item If $\po\subseteq R_{2}$ then $H_{\S}\left(R\right)=\left(\tppi,\po\right)$. 
\item Otherwise, either $R_{2}\subseteq\tppi\cup\ntppi\cup\eq$ or $R_{2}\subseteq\tpp\cup\ntpp\cup\eq$
and $\eq\subseteq R_{2}$ ($R_{2}\in\HH\cup\QH$). 
\begin{itemize}
\item If $\S_{1}=\QH$ then $R_{1}\subseteq\tppi\cup\ntppi$ and therefore
either $R_{2}\subseteq\tppi\cup\ntppi\cup\eq$ or $R_{2}=\eq$. Thus,
$H_{\S}\left(R\right)$ is $\left(\tppi,\tppi\right)$, $\left(\tppi,\ntppi\right)$,
or $\left(\tppi,\eq\right)$. 
\item If $\S_{2}=\QH$ then either $R_{2}\subseteq\tppi\cup\ntppi$ or $R_{2}=\eq$.
Thus, $H_{\S}\left(R\right)=\left(\tppi,\tppi\right)$ or $H_{\S}\left(R\right)=\left(\tppi,\ntppi\right)$
or $H_{\S}\left(R\right)=\left(\tppi,\eq\right)$. 
\end{itemize}
\end{itemize}
\item Otherwise $R_{1}\subseteq\ntpp\cup\eq\cup\ntppi$ : 
\begin{itemize}
\item If $\ntpp\subseteq R_{1}$ then $R_{1}=\ntpp$ ($R_{1}\in\HH\cup\QH$)
and $R_{2}\subseteq\tpp\cup\ntpp\cup\eq$. Thus, $H_{\S}\left(R\right)=\left(\ntpp,\tpp\right)$
or $H_{\S}\left(R\right)=\left(\ntpp,\ntpp\right)$ or $H_{\S}\left(R\right)=\left(\ntpp,\eq\right)$. 
\item If $\ntppi\subseteq R_{1}$ then $R_{1}=\ntppi$ ($R_{1}\in\HH\cup\QH$)
and $R_{2}\subseteq\tppi\cup\ntppi\cup\eq$. Thus, $H_{\S}\left(R\right)=\left(\ntppi,\tppi\right)$
or $H_{\S}\left(R\right)=\left(\ntppi,\ntppi\right)$ or $H_{\S}\left(R\right)=\left(\ntppi,\eq\right)$. 
\item If $\eq\subseteq R_{1}$ then $R_{1}=\eq$ ($R_{1}\in\HH\cup\QH$)
and $R_{2}\subseteq\po\cup\tpp\cup\ntpp\cup\eq\cup\tppi\cup\ntppi$.
Thus, $H_{\S}\left(R\right)_{1}=\eq$ and $H_{\S}\left(R\right)_{2}$
is either $\po$, $\tpp$, $\ntpp$, $\tppi$, $\ntppi$, or $\eq$. 
\end{itemize}
\end{itemize}
\end{proof}

Note the following proposition, before identifying the tractable subclasses.
It shows that although the tractable subalgebras of $\RCCH$ cannot
be combined to obtain algebraically tractable Cartesian subalgebras,
the algebraically consistent networks over the majority of these combinations
are satisfiable. 
\begin{proposition}
\label{prop:vois-acoherent-satisfiable}Let $\R$ be a topological
region space such that $\TRCCHv$ is a complete sequential formalism.
Let $\S$ be a subset of $\TRCCHv$ satisfying $\S_{i}\in\left\{ \HH,\QH\right\} $
and $\S_{i}=\HH\implies\left(i=m\ou\S_{i+1}=\QH\right)\and\left(i=1\ou\S_{i-1}=\QH\right)$
for all $i\in\I$.

Algebraically consistent networks over $\S$ are satisfiable. 
\end{proposition}

\begin{proof}
We apply the refinement theorem~\cite{cohen2017checking}. $H_{\S}$
is a refinement from $\S$ to the set $\Base\cup\left\{ \left(\vide,\ldots,\vide\right)\right\} $
(since $h_{\HH}$ is a refinement from $\HH\cup\QH$ to $\Base_{\RCCH}\cup\left\{ \vide\right\} $~\cite{renz1999maximal}).
$\S$ is algebraically stable by $H_{\S}$. Indeed, since on the one
hand, $\HH$ and $\QH$ are algebraically stable by $h_{\HH}$~\cite{renz1999maximal}.
Since, on the other hand, for any $\conv$-consistent relation $R\in\S$,
$H_{\S}\left(R\right)$ is $\conv$-consistent (by Lemma~\ref{lem:voisinage-stable}).
Algebraically consistent networks over $\Base\cup\left\{ \left(\vide,\ldots,\vide\right)\right\} $
are satisfiable ($\TRCCHv$ is complete). By the refinement theorem,
algebraically consistent networks over $\S$ are satisfiable. 
\end{proof}

Satisfiability of algebraically consistent networks is, in general,
a weaker property than algebraic tractability. It is not equivalent
for subclasses that are not $\conv$-closed. Indeed, applying the
algebraic closure on a network over a subclass which is not $\conv$-closed
can move the network out of the subclass. In that case, we cannot
therefore conclude that the network is satisfiable (if it is not trivially
inconsistent). The previous subclasses are not $\conv$-closed: the
projection of $\ntpp$, $\vois\ntpp=\tpp\cup\ntpp\cup\eq$ does not
belong to $\QH$. 

We end this section by showing that the subclasses of the following
particular forms $\left(\QH\times\Hntpp\right)^{\star}$, $\left(\Hntpp\times\QH\right)^{\star}$,
$\Hntpp\times\left(\QH\times\Hntpp\right)^{\star}$, and $\QH\times\left(\Hntpp\times\QH\right)^{\star}$
 are algebraically tractable. 
\begin{proposition}
\label{prop:classes-traitables-points}Let $\R$ be a topological
region space such that $\TRCCHv$ is a complete sequential formalism.
Let $\S$ be a subset of $\TRCCHv$ satisfying one of the two following
properties: 
\begin{itemize}
\item $\S_{i}=\Hntpp$ if $i$ is even and $\S_{i}=\QH$ otherwise, for
all $i\in\I$, 
\item $\S_{i}=\Hntpp$ if $i$ is odd and $\S_{i}=\QH$ otherwise, for all
$i\in\I$.
\end{itemize}
The subclass $\S$ is algebraically tractable. 
\end{proposition}

\begin{proof}
Let $\S$ be a subset of $\TRCCHv$ as described in the statement.
$\S$ satisfies the conditions of the first part of the refinement
theorem (see the proof of Proposition~\ref{prop:vois-acoherent-satisfiable}).
$\S$ also satisfies the conditions of the second part. Indeed, on
the one hand, $\S$ is a subclass (since $\S$ is Cartesian and each
$\S_{i}$ is a subclass~\cite{renz1999maximal} ; Lemma~\ref{lem:ss-classe}).
On the other hand, $\S$ is $\conv$-closed (by Lemma~\ref{lem:voisinage-clos}
and since $\S$ is a Cartesian subclass). $\S$ is thus algebraically
tractable~(refinement theorem~\cite{cohen2017checking}). 
\end{proof}

Note that the tractable subclasses identified by Proposition~\ref{prop:classes-traitables-points}
do not contain all the basic relations (thus, they are not subalgebras).

\section{Topological Sequences on a Partition of Time\label{sec:-temporaliser-sur-partition}}

We have shown that there is no algebraically tractable Cartesian subalgebra
in the context of $\TRCCHv$, the context of regions described at
different time points between which there are no intermediary relations
(i.e. at time points which characterize all the qualitative changes).
Does this mean that there are no large tractable subclasses in the
context of (topological) temporal sequences? We show that this is
not the case, by considering topological temporal sequences describing
the evolution of regions on a time partition (i.e. on a contiguous
alternation of instants and open intervals).

\subsection{Formalization}

We begin by defining the formalism of topological temporal sequences
on a partition of time, which we denote $\TRCCHd$. Without loss of
generality, we consider only the partitions of the interval $[t_{0},t_{l}[$
of the form $\left(t_{0},]t_{0},t_{1}[,\ldots,t_{l-1},]t_{l-1},t_{l}[\right)$
with $m=2l$, $t_{i-1}<t_{i}$ for all $i\in\left\{ 1,\ldots,l\right\} $,
$t_{i}\in\mathbb{R}$ for all $i\in\left\{ 0,\ldots,l\right\} $,
and $l\in\mathbb{N}^{*}$. Thus, the sequences of $\TRCCHd$ describe
the topological relations at each time periods $t_{i}$ and during
each interval $]t_{i},t_{i+1}[$. 
\begin{definition}
Let $(\R,T)$ be a topological region space\emph{.} Let $t_{0},\ldots,t_{l}\in\mathbb{R}$
be consecutive instants.

The formalism $\TRCCHd$ (associated with $(\R,T)$) is the triplet
$\left(\A,\U,\varphi\right)$ where:
\begin{itemize}
\item $\U$ is the set of continuous functions from $[t_{1},t_{m}[$ to
the set of regions $\R$, 
\item $\A$ is the multi-algebra whose Cartesian product is $\RCCAH^{m}$
and whose projections satisfy $\conv_{i}^{j}b=\upconv b$ if $\left|j-i\right|=1$
and $i$ is even, $\conv_{i}^{j}b=\downconv b$ if $\left|j-i\right|=1$
and $i$ is odd, and $\conv_{i}^{j}b=\B_{\RCCH}$ otherwise, for all
$b\in\Base_{\RCCH}$ and $i,j\in\I$ with $\upconv$ and $\downconv$
defined by Table~\ref{tab:Relations-dominantes-et}, and 
\item $\varphi$ is the function from $\A$ to $2^{\U\times\U}$ such that
for all $R\in\A$, $\varphi\left(R\right)$ is the set of pairs of
functions $\left(f,f'\right)\in\U\times\U$ satisfying at each instant
$t_{i}$ the relation $R_{2i+1}$ (i.e. $\forall i\in\left\{ 0,\ldots,l-1\right\} \ \left(f\left(t_{i}\right),f'\left(t_{i}\right)\right)\in\varphi_{\RCCH}\left(R_{2i+1}\right)$)
and satisfying, for each $i\in\left\{ 0,\ldots,l-1\right\} $, one
(and only one) basic relation $\ b\subseteq R_{2i+2}$ at each instant
between $t_{i}$ and $t_{i+1}$ (i.e. $\exists b\in\Base_{\RCCH}\ b\subseteq R_{2i+2}\ \forall t\in]t_{i},t_{i+1}[\ \left(f\left(t\right),f'\left(t\right)\right)\in\varphi_{\RCCH}\left(b\right)$),
with $\varphi_{\RCCH}$ the interpretation function of $\RCCH$.
\end{itemize}
\end{definition}

\begin{remark}
The operator $\upconv$ encodes the dominance graph of $\RCCH$ described
in Figure~\ref{fig:Graphe-de-voisinage}.b, i.e. the possible evolutions
of relations being satisfied during an open interval ($\upconv$ returns
the corresponding relations possibly satisfied at the limits of the
interval). The operators $\upconv$ and $\downconv$ enforces continuity
(when $\TRCCHd$ is a complete sequential formalism). Remark~\ref{sec:Discussion-sur-la}
on $\TRCCHv$ also applies to $\TRCCHd$. 
\end{remark}

\begin{example}
An example of relations of $\TRCCHd$, with $m=4$, is the relation
$R=\left(\tpp\cup\ntpp,\po\cup\eq,\ec\cup\dc,\dc\right)$. This relation
means that the first region is included in the second at the instant
$t_{0}$ ($R_{1}$ is satisfied), then they overlap during the interval
$]t_{0},t_{1}[$ or are equal during the interval $]t_{0},t_{1}[$
($R_{2}$ is satisfied), they are disjoined at the instant $t_{1}$
($R_{3}$ is satisfied), and finally they are disconnected at every
instant of $]t_{1},t_{2}[$ ($R_{4}$ is satisfied). The only satisfiable
basic relation included in $R$ is $\left(\tpp,\po,\ec,\dc\right)=\conv\left(R\right)$.

\begin{table}
\begin{centering}
{\footnotesize{}{}}%
\begin{tabular}{|c|c|c|}
\hline 
{\footnotesize{}{}$b$} & {\footnotesize{}{}$\upconv b$} & {\footnotesize{}{}$\downconv b$}\tabularnewline
\hline 
\hline 
{\footnotesize{}{}$\dc$} & {\footnotesize{}{}$\dc\cup\ec$} & {\footnotesize{}{}$\dc$}\tabularnewline
\hline 
{\footnotesize{}{}$\ec$} & {\footnotesize{}{}$\ec$} & {\footnotesize{}{}$\dc\cup\ec\cup\po$}\tabularnewline
\hline 
{\footnotesize{}{}$\po$} & {\footnotesize{}{}$\ec\cup\po\cup\tpp\cup\tppi\cup\eq$} & {\footnotesize{}{}$\po$}\tabularnewline
\hline 
{\footnotesize{}{}$\tpp$} & {\footnotesize{}{}$\tpp\cup\eq$} & {\footnotesize{}{}$\po\cup\tpp\cup\ntpp$}\tabularnewline
\hline 
{\footnotesize{}{}$\ntpp$} & {\footnotesize{}{}$\ntpp\cup\tpp\cup\eq$} & {\footnotesize{}{}$\ntpp$}\tabularnewline
\hline 
{\footnotesize{}{}$\eq$} & {\footnotesize{}{}$\eq$} & {\footnotesize{}{}$\B_{\RCCH}\backslash\left(\dc\cup\ec\right)$}\tabularnewline
\hline 
\end{tabular}
\par\end{centering}
\caption{\label{tab:Relations-dominantes-et}Dominant and dominated relations
of the basic relations of $\protect\RCCH$.}
\end{table}
\end{example}

\subsection{Tractability Results}

We now identify large algebraically tractable Cartesian subalgebras,
by applying again the refinement theorem. We begin by showing its
conditions. 
\begin{lemma}
\label{lem:dominance-stable}Let $\S\in\{\HH\times\HH,\HH\times\QH,\QH\times\HH,\QH\times\QH,\HH\times\CH,\QH\times\CH\}$
be a subclass of $\TRCCHd$ $(m=2)$ and $R\in\S$.

If $R$ is $\conv$-consistent then $H_{\S}\left(R\right)$ is $\conv$-consistent. 
\end{lemma}

\begin{proof}
From the definitions of $\HH$, $\QH$, $\CH$, $h_{\HH}$, $h_{\CH}$,
and the projections of $\TRCCHd$, we derive the lemma. Let $\S\in\{\HH\times\HH,\HH\times\QH,\QH\times\HH,\QH\times\QH,\HH\times\CH,\QH\times\CH\}$
and $R\in\S$ be a $\conv$-consistent relation. We show that $H_{\S}\left(R\right)$
is $\conv$-consistent. 
\begin{itemize}
\item If $\dc\subseteq R_{1}$ then $\dc\subseteq R_{2}$ and therefore
$H_{\S}\left(R\right)=\left(\dc,\dc\right)$. 
\item Otherwise, if $\ec\subseteq R_{1}$ then $R_{2}\cap\left(\dc\cup\ec\cup\po\right)\neq\vide$. 
\begin{itemize}
\item If $\S_{2}\in\left\{ \HH,\QH\right\} $ then $H_{\S}\left(R\right)_{1}=\ec$
and $H_{\S}\left(R\right)_{2}\subseteq\dc\cup\ec\cup\po$. 
\item Suppose $\S_{2}=\CH$. 
\begin{itemize}
\item If $\dc\nsubseteq R_{2}$ and $\po\nsubseteq R_{2}$ then $\ec\subseteq R_{2}$
and therefore $R_{2}=\ec$ (since $R_{2}\in\CH$). Thus, $H_{\S}\left(R\right)=\left(\ec,\ec\right)$ 
\item If $\dc\subseteq R_{2}$ or $\po\subseteq R_{2}$ then $H_{\S}\left(R\right)_{1}=\ec$
and $H_{\S}\left(R\right)_{2}\subseteq\dc\cup\po$. 
\end{itemize}
\end{itemize}
\item Otherwise, if $\po\subseteq R_{1}$ then $\po\subseteq R_{2}$ and
$R_{2}\cap\left(\dc\cup\ec\right)=\vide$ (since $R_{1}\cap\left(\dc\cup\ec\right)=\vide$).
Therefore, $H_{\S}\left(R\right)=\left(\po,\po\right)$. 
\item Otherwise, if $\tpp\subseteq R_{1}$ then $R_{2}\cap\left(\tpp\cup\ntpp\cup\po\right)\neq\vide$
and $R_{2}\cap\left(\dc\cup\ec\right)=\vide$. 
\begin{itemize}
\item Suppose $\S_{2}\in\left\{ \HH,\QH\right\} $. 
\begin{itemize}
\item If $\po\subseteq R_{2}$ or $\tpp\subseteq R_{2}$ then $H_{\S}\left(R\right)=\left(\tpp,\po\right)$
or $H_{\S}\left(R\right)=\left(\tpp,\tpp\right)$. 
\item If $\po\nsubseteq R_{2}$ and $\tpp\nsubseteq R_{2}$ then $\ntpp\subseteq R_{2}$
and therefore $R_{2}=\ntpp$ (since $R_{2}\in\HH\cup\QH$). Thus,
$H_{\S}\left(R\right)=\left(\tpp,\ntpp\right)$. 
\end{itemize}
\item Suppose $\S_{2}=\CH$. 
\begin{itemize}
\item If $\po\subseteq R_{2}$ or $\ntpp\subseteq R_{2}$ then $H_{\S}\left(R\right)=\left(\tpp,\po\right)$
or $H_{\S}\left(R\right)=\left(\tpp,\ntpp\right)$. 
\item If $\po\nsubseteq R_{2}$ and $\ntpp\nsubseteq R_{2}$ then $\tpp\subseteq R_{2}$
and therefore $R_{2}\subseteq\tpp\cup\eq$ ($R_{2}\in\CH$). Thus,
$H_{\S}\left(R\right)=\left(\tpp,\tpp\right)$. 
\end{itemize}
\end{itemize}
\item Otherwise, if $\tppi\subseteq R_{1}$ then $R_{2}\cap\left(\tppi\cup\ntppi\cup\po\right)\neq\vide$
and $R_{2}\cap\left(\dc\cup\ec\right)=\vide$. 
\begin{itemize}
\item If $\po\subseteq R_{2}$ then $H_{\S}\left(R\right)=\left(\tppi,\po\right)$. 
\item If $\po\nsubseteq R_{2}$ then $R_{2}\subseteq\tppi\cup\ntppi\cup\eq$
($R_{2}\in\HH\cup\QH\cup\CH$). 
\begin{itemize}
\item If $\S_{2}\in\left\{ \HH,\QH\right\} $ and $\tppi\subseteq R_{2}$,
we have $H_{\S}\left(R\right)=\left(\tppi,\tppi\right)$. 
\item If $\S_{2}\in\left\{ \HH,\QH\right\} $ and $\tppi\nsubseteq R_{2}$,
we have $\eq\nsubseteq R_{2}$ (since $R_{2}\in\HH\cup\QH$). Therefore,
$H_{\S}\left(R\right)=\left(\tppi,\ntppi\right)$. 
\item If $\S_{2}\in\CH$, since $R_{2}\cap\left(\tppi\cup\ntppi\right)\neq\vide$,
we have $H_{\S}\left(R\right)=\left(\tppi,\ntppi\right)$ or $H_{\S}\left(R\right)=\left(\tppi,\tppi\right)$. 
\end{itemize}
\end{itemize}
\item Otherwise, $R_{1}\subseteq\ntpp\cup\ntppi\cup\eq$ and $R_{2}\subseteq\po\cup\tpp\cup\ntpp\cup\tppi\cup\ntppi\cup\eq$. 
\begin{itemize}
\item If $\ntpp\subseteq R_{1}$ then $R_{1}=\ntpp$ ($R_{1}\in\HH\cup\QH$)
and therefore $R_{2}=\ntpp$. Thus, $R=H_{\S}\left(R\right)=\left(\ntpp,\ntpp\right)$. 
\item If $\ntppi\subseteq R_{1}$ then, similarly, $R=H_{\S}\left(R\right)=\left(\ntppi,\ntppi\right)$. 
\item Otherwise, $\eq=R_{1}$. Therefore, $H_{\S}\left(R\right)_{1}=\eq$
and $H_{\S}\left(R\right)_{2}\subseteq\po\cup\tpp\cup\ntpp\cup\tppi\cup\ntppi\cup\eq$. 
\end{itemize}
\end{itemize}
\end{proof}

Note that, as in the context of $\TRCCHv$, algebraically consistent
networks over most combinations of the subalgebras $\QH$ and $\HH$,
but \emph{also} of $\CH$, are satisfiable. 
\begin{proposition}
\label{prop:dominance-reseau-acoherent-satisfiable}Let $\R$ be a
topological region space such that $\TRCCHd$ is a complete sequential
formalism. Let $\S$ be a subset of $\TRCCHd$ satisfying $S_{2i-1}\in\left\{ \HH,\QH\right\} $
and $\S_{2i}\in\left\{ \HH,\QH,\CH\right\} $ for all $i\in\left\{ 1,\ldots,l\right\} $.

Algebraically consistent networks over $\S$ are satisfiable. 
\end{proposition}

\begin{proof}
We apply the refinement theorem~\cite{cohen2017checking}. $H_{\S}$
is a refinement from $\S$ to the set $\Base\cup\left\{ \left(\vide,\ldots,\vide\right)\right\} $
(since $h_{\HH}$ (resp. $h_{\CH}$) is a refinement from $\HH\cup\QH$
(resp. $h_{\CH}$) to $\Base_{\RCCH}\cup\left\{ \vide\right\} $~\cite{renz1999maximal}).
$\S$ is algebraically stable by $H_{\S}$. Indeed, since on the one
hand, $\HH$ (resp. $\QH$ ; resp. $\CH$) is algebraically stable
by $h_{\HH}$ (resp. $h_{\HH}$ ; resp. $h_{\CH}$)~\cite{renz1999maximal}.
Since, on the other hand, for any $\conv$-consistent relation $R\in\S$,
$H_{\S}\left(R\right)$ is $\conv$-consistent (by Lemma~\ref{lem:dominance-stable}).
Algebraically consistent networks over $\Base\cup\left\{ \left(\vide,\ldots,\vide\right)\right\} $
are satisfiable ($\TRCCHd$ is complete). By the refinement theorem,
algebraically consistent networks over $\S$ are satisfiable. 
\end{proof}

\begin{lemma}
\label{lem:dominance-clos}Let $r\in\RCCAH\backslash\Nrcc$. We have: 
\begin{itemize}
\item $\upconv r\in\HH$, 
\item $\downconv r\in\HH\cap\QH\cap\CH$. 
\end{itemize}
\end{lemma}

\begin{proof}
From the definitions of $\HH$, $\QH$, $\CH$ and the projections
of $\TRCCHd$, we derive the lemma. Let $r\in\RCCAH\backslash\Nrcc$
(see Table~\ref{tab:Definitions-de-Q8,H8,C8}). On the one hand,
we show $\upconv r\in\HH$. For this, we show the three following
properties: $\upconv r\in\RCCAH\backslash\Nrcc$, $\ntpp\subseteq\upconv r\implique\tpp\subseteq\upconv r$,
and $\upconv r\notin V=\{\ec\cup\ntpp\cup\eq,\dc\cup\ec\cup\ntpp\cup\eq,\ec\cup\ntppi\cup\eq,\dc\cup\ec\cup\ntppi\cup\eq\}$.
We show $\upconv r\in\RCCAH\backslash\Nrcc$. If $\upconv r\cap\left(\tpp\cup\ntpp\right)=\vide$
or $\upconv r\cap\left(\tppi\cup\ntppi\right)=\vide$ then $\upconv r\in\RCCAH\backslash\Nrcc$.
Suppose $\upconv r\cap\left(\tpp\cup\ntpp\right)\neq\vide$ and $\upconv r\cap\left(\tppi\cup\ntppi\right)\neq\vide$.
Therefore, $r\cap\left(\po\cup\tpp\cup\ntpp\right)\neq\vide$ and
$r\cap\left(\po\cup\tppi\cup\ntppi\right)\neq\vide$. If $r\cap\left(\tpp\cup\ntpp\right)\neq\vide$
and $r\cap\left(\tppi\cup\ntppi\right)\neq\vide$ then $\po\subseteq r$
(since $r\in\RCCAH\backslash\Nrcc$). Thus, $\po\subseteq r$ and
therefore $\po\subseteq\upconv r$. We have $\upconv r\in\RCCAH\backslash\Nrcc$.
We show $\ntpp\subseteq\upconv r\implique\tpp\subseteq\upconv r$.
If $\ntpp\subseteq\upconv r$, then $\ntpp\subseteq r$ and therefore
$\upconv\ntpp\subseteq\upconv r$. Thus, $\tpp\subseteq\upconv r$.
We show $\upconv r\notin V$. If $\upconv r\cap\left(\ntpp\cup\ntppi\right)=\vide$
then $\upconv r\notin V$. If $\ntpp\subseteq\upconv r$, then $\tpp\subseteq\upconv r$.
If $\ntppi\subseteq\upconv r$ then $\tppi\subseteq\upconv r$. Thus,
in all cases, $\upconv r\notin V$.

On the other hand, we show $\downconv r\in\HH\cap\QH\cap\CH$. If
$\downconv r\cap\left(\tpp\cup\ntpp\right)=\vide$ or $\downconv r\cap\left(\tppi\cup\ntppi\right)=\vide$
then $\downconv r\in\RCCAH\backslash\Nrcc$. Suppose $\downconv r\cap\left(\tpp\cup\ntpp\right)\neq\vide$
and $\downconv r\cap\left(\tppi\cup\ntppi\right)\neq\vide$. We have
$r\cap\left(\tpp\cup\ntpp\cup\eq\right)\neq\vide$ and $r\cap\left(\tppi\cup\ntppi\cup\eq\right)\neq\vide$.
If $\eq\subseteq r$, then $\po\subseteq\downconv r$ and therefore
$\downconv r\in\RCCAH\backslash\Nrcc$. If $\eq\nsubseteq r$, then
$r\cap\left(\tpp\cup\ntpp\right)\neq\vide$ and $r\cap\left(\tppi\cup\ntppi\right)\neq\vide$.
Since $r\in\RCCAH\backslash\Nrcc$, $\po\subseteq r$. Thus, $\po\subseteq\downconv r$
and therefore $\downconv r\in\RCCAH\backslash\Nrcc$. Thus, in all
cases, $\downconv r\in\RCCAH\backslash\Nrcc$. Moreover, we have $\downconv r\in\Prcc$
since $\downconv r\notin V$. Indeed, if $\ec\subseteq\downconv r$,
then $\ec\subseteq r$ and therefore $\po\subseteq\downconv r$. By
the same argument, we have $\downconv r\in\CH$. In addition, we have
$\downconv r\in\HH\cap\QH$ and therefore $\downconv r\in\HH\cap\QH\cap\CH$,
since if $\eq\subseteq\downconv r$ then $\eq\subseteq r$ and therefore
$\po\cup\tpp\cup\tppi\subseteq\downconv r$. 
\end{proof}

We end by showing that the subalgebras of the form $\left(\HH\times\left\{ \HH,\QH,\CH\right\} \right)^{\star}$
are algebraically tractable. 
\begin{theorem}
Let $\R$ be a topological region space such that $\TRCCHd$ is a
complete sequential formalism.

Subalgebras $\S$ of $\TRCCHd$ satisfying $S_{2i-1}=\HH$ and $\S_{2i}\in\left\{ \HH,\QH,\CH\right\} $
for all $i\in\left\{ 1,\ldots,l\right\} $ are algebraically tractable. 
\end{theorem}

\begin{proof}
Let $\S$ be a subset of $\TRCCHd$ satisfying $S_{2i-1}=\HH$ and
$\S_{2i}\in\left\{ \HH,\QH,\CH\right\} $ for all $i\in\left\{ 1,\ldots,l\right\} $.
$\S$ satisfies the conditions of the first part of the refinement
theorem (i.e. $\S$ satisfies the conditions of the first implication
; see the proof of Proposition~\ref{prop:dominance-reseau-acoherent-satisfiable}).
$\S$ also satisfies the conditions of the second part (i.e. the conditions
of the second implication). Indeed, on the one hand, $\S$ is a subclass
(since $\S$ is Cartesian and each $\S_{i}$ is a subclass~\cite{renz1999maximal}).
On the other hand, $\S$ is $\conv$-closed (by Lemma~\ref{lem:dominance-clos}
and since $\S$ is a Cartesian subclass). $\S$ is thus algebraically
tractable~(refinement theorem~\cite{cohen2017checking}). 
\end{proof}

\section{Conclusion}

First, we have focused on $\TRCCHv$, the qualitative formalism of
topological temporal sequences describing the evolution of regions
at instants between which there are no intermediary relations (i.e.
at time points which characterize all the qualitative changes). We
have shown that there is no algebraically tractable Cartesian subalgebra
(subclass containing all basic relations) for $\TRCCHv$ when the
length of sequences is longer than $3$. However, we have identified
some tractable subclasses. The price of tractability has been to give
up the relations containing $\ntpp$ not containing $\tpp$ and thus
to give up the basic relation $\ntpp$.

Then, we have formalized $\TRCCHd$, the qualitative formalism of
topological temporal sequences describing the evolution of regions
on a partition of time (i.e. on a contiguous alternation of instants
and open intervals). In this context, we have identified large algebraically
tractable Cartesian subalgebras.

It is possible to identify other algebraically tractable subclasses
for $\TRCCHv$ and $\TRCCHd$. The tractability limit of the subclasses
of these two formalisms should be precisely determined. In particular,
a definitive answer to the question of the existence of polynomial
Cartesian subalgebra for $\TRCCHv$ should be given. Note that the
identification of universes ensuring the completeness of $\TRCCHv$
and of $\TRCCHd$ remains an open problem, on which we are working.

Concerning the applications, $\TRCCHv$ and $\TRCCHd$ can be used
to decide if it is possible to go from a topological scenario $S$
to another $S'$, with at most $m$ qualitative changes, while satisfying
at each instant the constraints of a network $N$ and to determine
one of the corresponding intermediate temporal sequences. This problem
should be useful for spatial planning. When $S$, $S'$, and $N$
correspond to one of the previous algebraically tractable subclasses
(for instance when $S$, $S'$, and $N$ are over $\HH$), the problem
is polynomial. Otherwise, it is possible that using tractable subclasses
still speeds up the resolution of the problem, as in the classic case~\cite{renz2001efficient}.
Note that $\TRCCHd$ is more interesting than $\TRCCHv$ for this
problem since it allows to find a more expressive intermediate sequence
while having larger tractable subclasses.

\bibliographystyle{splncs04}
\bibliography{biblio}
 
\end{document}